\documentclass{article}
\usepackage{amsmath,amssymb,amsthm, amsfonts}
\usepackage{calligra,indentfirst,epsfig}
\usepackage{enumerate}
\usepackage{indentfirst}
\usepackage{mathrsfs}
\usepackage{pstricks,graphicx}
\usepackage{setspace}
\usepackage{latexsym}
\usepackage{yfonts}
\usepackage[all]{xy}
\usepackage{array}
\usepackage{algorithm}
\usepackage[noend]{algpseudocode}
\makeatletter
\def\BState{\State\hskip-\ALG@thistlm}
\makeatother
\usepackage{rotating}
\newcommand{\noop}[1]{}
\calligra
\def\NN{\mathbb N}

\def\RR{\mathbb R}
\def\PP{\mathcal P}

\def\HH{\mathcal H}

\def\LL{\mathscr L}

\def\xx{\mathbf{x}}
\def\yy{\mathbf{y}}
\def\zz{\mathbf{z}}

\newcommand{\la} {\lambda}
\newcommand{\al} {\alpha}
\newtheorem{theorem}{Theorem}[section]

\newtheorem{definition}{Definition}[section]
\newtheorem{proposition}{Proposition}[section]
\newtheorem{assumption}{Assumption}[section]
\newtheorem{remark}{Remark}[section]
\allowdisplaybreaks
\begin{document}
\title{Manifold regularization based on Nystr{\"o}m type subsampling}
\author{Abhishake Rastogi{\footnote{Corresponding Author, Email address: abhishekrastogi2012@gmail.com}}, Sivananthan Sampath{\footnote{Email address: siva@maths.iitd.ac.in}}\\
{\it Department of Mathematics}\\{\it Indian Institute of
Technology Delhi}\\{\it New Delhi 110016, India}}
\date{}
\maketitle
\begin{abstract}
In this paper, we study the Nystr{\"o}m type subsampling for large scale kernel methods to reduce the computational complexities of big data. We discuss the multi-penalty regularization scheme based on Nystr{\"o}m type subsampling which is motivated from well-studied manifold regularization schemes. We develop a theoretical analysis of multi-penalty least-square regularization scheme under the general source condition in vector-valued function setting, therefore the results can also be applied to multi-task learning problems. We achieve the optimal minimax convergence rates of multi-penalty regularization using the concept of effective dimension for the appropriate subsampling size. We discuss an aggregation approach based on linear function strategy to combine various Nystr{\"o}m approximants. Finally, we demonstrate the performance of multi-penalty regularization based on Nystr{\"o}m type subsampling on Caltech-101 data set for multi-class image classification and NSL-KDD benchmark data set for intrusion detection problem.
\end{abstract}
{\bf Keywords:} Multi-task learning; Manifold learning; Multi-penalty regularization; Nystr{\"o}m type subsampling; Optimal rates; Linear functional strategy.\\
{\bf  Mathematics Subject Classification 2010:} 68T05, 68Q32.
\section{Introduction}\label{Introduction}
Multi-task learning is an approach which learns multiple tasks simultaneously. The problem has potential to learn the structure of the related tasks. The idea is that exploring task relatedness can lead to improved performance. Various learning algorithms are studied to incorporate the structure of task relations in literature \cite{Argyriou,Ciliberto1,Ciliberto,Evgeniou1}. In agreement with past empirical work on multi-task learning, learning multiple related tasks simultaneously has been empirically \cite{Ando,Bakker,Evgeniou1,Evgeniou07,Jebara,Torralba,Yu,Zhang05} and theoretically \cite{Ando,Baxter,Ben} shown significantly improved performance relative to learning each task independently. Multi-task learning is becoming interesting due to its applications in computer vision, image processing and many other fields such as object detection/classification \cite{Minh}, image denoising, inpainting, finance and economics forecasting predicting \cite{Greene}, marketing modeling for the preferences of many individuals \cite{Allenby,Arora} and in bioinformatics for example to study tumor prediction from multiple micro–array data sets or analyze data from multiple related diseases.

In this work, we discuss multi-task learning approach that considers a notion of relatedness based on the concept of manifold regularization. In scalar-valued function setting, Belkin et al. \cite{Belkin} introduced the concept of manifold regularization which focus on a semi-supervised framework that incorporates labeled and unlabeled data in a general-purpose learner. Minh and Sindhwani \cite{Minh11} generalized the concept of manifold learning for vector-valued functions which exploits output inter-dependencies while enforcing smoothness with respect to input data geometry. Further, Minh and Sindhwani \cite{Minh} present a general learning framework that encompasses learning across three different paradigms, namely vector-valued, multi-view and semi-supervised learning simultaneously. Multi-view learning approach is considered to construct the regularized solution based on different views of the input data using different hypothesis spaces \cite{Brefeld,Luo13,Luo,Minh,Rosenberg,Sindhwani,Sun}. Micchelli and Pontil \cite{Micchelli1} introduced the concept of vector-valued reproducing kernel Hilbert spaces to facilitate theory of multi-task learning. Also, the fact that every vector-valued RKHS is corresponding to some operator-valued positive definite kernel, reduces the problem of choosing appropriate RKHS (hypothesis space) to choosing appropriate kernel \cite{Micchelli1}. In paper \cite{Bach,Bucak,Micchelli}, the authors proposed multiple kernel learning from a set of kernels. Here we consider the direct sum of reproducing kernel Hilbert spaces as the hypothesis space.

Multi-task learning is studied under the elegant and effective framework of kernel methods. The expansion of automatic data generation and acquisition bring data of huge size and complexity which raises challenges to computational capacities. In order to tackle these difficulties, various techniques are discussed in the literature such as replacing the empirical kernel matrix with a smaller matrix obtained by (column) subsampling \cite{Bach13,Rudi,Williams}, greedy-type algorithms \cite{Smola}, divide-and-conquer approach \cite{Guo17,Guo,Zhang15}. We are inspired from the work of Kriukova et al. \cite{Kriukova16} in which the authors discussed an approach to aggregate various regularized solutions based on Nystr{\"o}m subsampling in single-penalty regularization. Here we consider the so-called Nystr{\"o}m type subsampling in large scale kernel methods for dealing with big data which particularly can be seen as a regularized projection scheme. We achieve the optimal convergence rates for multi-penalty regularization based on the Nystr{\"o}m type subsampling approach, provided the subsampling size is appropriate. We adapt the aggregation approach for multi-task learning manifold regularization scheme to improve the accuracy of the results. We consider the linear combination of Nystr{\"o}m approximants and try to obtain a combination of the approximants which is closer to the target function. The coefficients of the linear combination are estimated by means of the linear functional strategy. The aggregation approach tries to accumulate the information hidden inside various approximants to produce the estimator of the target function \cite{Kriukova,Kriukova16} (also see reference therein).

The paper is organized as follows. In Section 2, we describe the framework of vector-valued multi-penalized learning problem with some basic definitions and notations. In Section 3, we discuss the convergence issues of the vector-valued multi-penalty regularization scheme based on Nystr{\"o}m type subsampling in the norm in $\LL^2_{\rho_X}$ and the norm in $\HH$. In Section 4, we discuss the aggregation approach to accumulate various estimators based on the Nystr{\"o}m type subsampling. In the last section, we demonstrate the performance of multi-penalty regularization based on Nystr{\"o}m type subsampling on Caltech-101 data set for multi-class image classification and NSL-KDD benchmark data set for intrusion detection problem.

\section{Multi-task learning via vector-valued RKHS}
The problem of learning multiple tasks jointly can be modeled by the vector-valued functions $f:X\to \RR^T$ whose components represent individual task-predictors, i.e., $f=(f_1,\ldots,f_T)$ for $f_t:X \to \RR$ ($1\leq t \leq T$). Here we consider general framework of vector-valued functions $f:X\to Y$ developed by Micchelli and Pontil \cite{Micchelli1} to address the multi-task learning algorithm. We
consider the concept of vector-valued reproducing kernel Hilbert space which is the extension of well-known scalar-valued reproducing kernel Hilbert space.

\begin{definition}{\bf Vector-valued reproducing kernel Hilbert space (RKHSvv).}
Let $X$ be a non-empty set, $(Y,\langle\cdot,\cdot\rangle_Y)$ be a real separable Hilbert space. The Hilbert space of functions from $X$ to $Y$ is called reproducing kernel Hilbert space if for any $x\in X$ and $y\in Y$, the linear functional which maps $f\in \HH$ to $\langle y,f(x)\rangle_Y$ is continuous.
\end{definition}

Suppose $\mathcal{L}(Y)$ be the Banach space of bounded linear operators on $Y$. A function $K:X\times X\to \mathcal{L}(Y)$ is said to be an operator-valued positive definite kernel if for each pair $(x,z)\in X\times X$, $K(x,z)^*=K(z,x)$, and for every finite set of points $\{x_i\}_{i=1}^N\subset X$ and $\{y_i\}_{i=1}^N\subset Y$,
$$\sum\limits_{i,j=1}^N\langle y_i,K(x_i,x_j)y_j\rangle_Y\geq 0.$$

There exists a unique Hilbert space $(\HH,\langle\cdot,\cdot\rangle_{\HH})$ of functions on  $X$ satisfying the following conditions:
\begin{enumerate}[(i)]
\item for all $x\in X$ and $y\in Y$, the functions $K_xy=K(\cdot,x)y\in \HH$, defined by
$$(K_xy)z=K(z,x)y \text{ for all } z\in X,$$
\item the span of the set $\{K_xy:x\in X, y\in Y\}$ is dense in $\HH$, and
\item for all $f\in \HH$, $\langle f(x),y\rangle_Y=\langle f,K_xy\rangle_{\HH}$ (reproducing property).
\end{enumerate}

Moreover, there is one to one correspondence between operator-valued positive definite kernels and vector-valued RKHS \cite{Micchelli1}.

In the learning theory, we are given with the random samples $\{(x_i,y_i):1\leq i \leq m\}$ drawn identically and independently from a unknown joint probability measure $\rho$ on the sample space $X\times Y$. We assume that the input space $X$ is a locally compact countable Hausdorff space and the output space $(Y,\langle\cdot,\cdot\rangle_Y)$ is a real separable space.  The goal is to predict the output values for the inputs. Suppose we predict $y$ for the input $x$ based on our algorithm but the true output is $y'$. Then we suffer a loss $\ell(y,y')$, where the loss function $\ell:Y\times Y\to \RR^+$. A widely used approach based on the square loss function $\ell(y,y')=||y-y'||_Y^2$ in regularization theory is Tikhonov type regularization:
$$\frac{1}{m}\sum\limits_{i=1}^m||f(x_i)-y_i||_Y^2+\la||f||_\HH^2,$$
The regularization parameter $\la$ controls the trade off between the error measuring the fitness of data and the complexity of the solution measured in the RKHS-norm.

We discuss the multi-task learning approach that considers a notion of task relatedness based on the concept of manifold regularization. In this approach, different RKHSvv are used to estimate the target functions based on different views of input data, such as different features or modalities and a data-dependent regularization term is used to enforce consistency of output values from different views of the same input example.

We consider the following regularization scheme to analyze the multi-task manifold learning scheme corresponding to different views:
\begin{equation}\label{manifold.funcl}
\mathop{\text{arg}\min}_{f \in \HH_{K_1}\bigoplus\ldots\bigoplus\HH_{K_v}} \left\{\frac{1}{m}\sum\limits_{i=1}^m||f(x_i)-y_i||_Y^2+\la_A||f||_\HH^2+\la_I\langle \mathbf{f},M\mathbf{f}\rangle_{Y^n}\right\},
\end{equation}
where $\{(x_i,y_i)\in X \times Y:1 \leq i \leq m\}\bigcup \{x_i \in X:m < i \leq n\}$ is given set of labeled and unlabeled data, $M$ is a symmetric, positive operator, $\la_A, \la_I \geq 0$ and $\mathbf{f}=(f(x_1),\ldots,f(x_n))^T \in Y^n$.

The direct sum of reproducing kernel Hilbert spaces $\HH=\HH_{K_1}\bigoplus\ldots\bigoplus\HH_{K_v}$ is also a RKHS. Suppose $K$ is the kernel corresponding to RKHS $\HH$.

Throughout this paper we assume the following hypothesis:
\begin{assumption}Let $\HH$ be a reproducing kernel Hilbert space of functions $f:X\to Y$ such that
\begin{enumerate}[(i)]
\item For all $x\in X$, $K_x:Y\to\HH$ is a Hilbert-Schmidt operator and $\kappa:=\sqrt{\sup\limits_{x \in X}Tr(K_x^*K_x)}<\infty$, where for Hilbert-Schmidt operator $A\in \mathcal{L}(\HH)$, $Tr(A):=\sum\limits_{k=1}^\infty\langle Ae_k,e_k\rangle$ for an orthonormal basis $\{e_k\}_{k=1}^\infty$ of $\HH$.
\item The real-valued function $\phi:X\times X \to \RR$, defined by $\phi(x,t)=\langle K_tv,K_xw\rangle_\HH$, is measurable $\forall  v,w\in Y$.
\end{enumerate}
\end{assumption}

By the representation theorem \cite{Minh}, the solution of the multi-penalized regularization problem (\ref{manifold.funcl}) will be of the form:
\begin{equation}\label{represnter.thm}
f_{\zz,\la}=\sum\limits_{i=1}^n K_{x_i}c_i, \text{  for some  } \mathbf{c}=(c_1,\ldots,c_n)=(\mathbb{J}\mathbb{K}_{n}+\la_A m\mathbb{I}_n+\la_I mL\mathbb{K}_{n})^{-1} \yy_n,
\end{equation}

where $(\mathbb{K}_{n})_{ij} = K(x_i,x_j)$ with $i,j \in \{1,\ldots,n\}$, $\mathbb{J} = \text{diag}(1, \ldots, 1,0, \ldots, 0)$ is $n\times n$ diagonal matrix with the first $m$ diagonal entries as $1$ and the rest $0$, $\mathbb{I}_n$ is identity of order $n$ and $\yy_n=[y_1,\ldots,y_m,0,\ldots,0]^T\in Y^n$.

In order to obtain the computationally efficient algorithm from the functional (\ref{manifold.funcl}), we consider the Nystr{\"o}m type subsampling which uses the idea of replacing the empirical kernel matrix with a smaller matrix obtained by (column) subsampling \cite{Kriukova16,Smola,Williams}. This can also be seen as a restriction of the optimization functional (\ref{manifold.funcl}) over the space:
\begin{equation*}
\HH^{\xx_s}:=\{f|f=\sum\limits_{i=1}^sK_{x_i}c_i,\mathbf{c}=(c_1,\ldots,c_s)\in Y^s\},
\end{equation*}
where $s\ll n$ and $\xx_s=(x_1,\ldots, x_s)$ is a subset of the input points in the training set.

The minimizer of the manifold regularization scheme (\ref{manifold.funcl}) over the space $\HH^{\xx_s}$ will be of the form:
\begin{equation}\label{Nys.approx}
f_{\zz,\la}^s=\sum\limits_{i=1}^s K_{x_i}c_i, \text{ for } \mathbf{c}=(c_1,\ldots,c_s)=(\mathbb{K}_{ms}^T\mathbb{K}_{ms}+\la_A m\mathbb{K}_{ss}+\la_I m\mathbb{K}_{ns}^TL\mathbb{K}_{ns})^\dagger \mathbb{K}_{ms}^T\yy,
\end{equation}
where $A^\dagger$ denotes the Moore-Penrose pseudoinverse of a matrix $A$, $(\mathbb{K}_{ms})_{ij} = K(x_i,\tilde{x}_j)$,
$(\mathbb{K}_{ss})_{kj} = K(\tilde{x}_k,\tilde{x}_j)$ with $i \in \{1,\ldots,m\}$ and $j, k \in \{1,\ldots,s\}$ and $\yy=[y_1,\ldots,y_m]^T\in Y^m$.

The computational time of the Nystr{\"o}m approximation (\ref{Nys.approx}) is of order $\mathcal{O}(sn^2)$ while the computational time complexity of standard manifold regularized solution (\ref{represnter.thm}) is of order $\mathcal{O}(n^3)$. Therefore, the randomized subsampling methods can break the memory barriers and consequently achieve much better time complexity compare with standard manifold regularization algorithm.

We analyze the more general vector-valued multi-penalty regularization scheme based on Nystr{\"o}m subsampling:
\begin{equation}\label{multi.plty.funcl1}
f_{\zz,\la}^s=\mathop{\text{arg}\min}_{f \in \HH^{\xx_s}} \left\{\frac{1}{m}\sum\limits_{i=1}^m||f(x_i)-y_i||_Y^2+\la_0||f||_\HH^2+\sum\limits_{j=1}^p\la_j||B_jf||_\HH^2\right\},
\end{equation}
where $B_j:\HH\to\HH~(1\leq j\leq p)$ are bounded operators, $\la_0>0$, $\la_j~(1\leq j\leq p)$ are non-negative real numbers and $\la$ denotes the ordered set $(\la_0,\la_1,\ldots,\la_p)$.

Here we introduce the sampling operator which is useful in the analysis of regularization schemes.
\begin{definition}
The {\bf sampling operator} $S_{\xx}:\HH \to Y^m$ associated with a discrete subset $\xx=(x_i)_{i=1}^m$ is defined by
$$S_{\xx}(f)=(f(x))_{x \in \xx}.$$
\end{definition}
Then its adjoint is given by
$$S_{\xx}^*\yy=\frac{1}{m}\sum\limits_{i=1}^m K_{x_i} y_i,~~~~\forall \yy=(y_1,\ldots,y_m)\in Y^m.$$
The sampling operator is bounded by $\kappa$.

We obtain the following explicit expression of the minimizer of the regularization scheme (\ref{multi.plty.funcl1}). The proof of the theorem follows the same steps as of Lemma 1 \cite{Rudi}.
\begin{theorem}\label{optimizer}
For the positive choice of $\la_0$, the functional (\ref{multi.plty.funcl1}) has unique minimizer:
\begin{equation*}\label{fzl}
f_{\zz,\la}^s=\left(P_{\xx_s}S_{\xx}^*S_{\xx}P_{\xx_s}+\la_0I+\sum\limits_{j=1}^p\la_jP_{\xx_s}B_j^*B_jP_{\xx_s}\right)^{-1} P_{\xx_s}S_{\xx}^*\yy,
\end{equation*}
where $P_{\xx_s}$ is the orthogonal projection operator with range $\HH^{\xx_s}$.
\end{theorem}

The data-free version of the considered regularization scheme (\ref{multi.plty.funcl1}) is
\begin{equation}\label{fl.funl}
f_\la^s:=\mathop{\text{arg}\min}_{f \in \HH^{\xx_s}} \left\{\int_Z||f(x)-y||_Y^2d\rho(x,y)+\la_0||f||_\HH^2+\sum\limits_{j=1}^p\la_j||B_jf||_\HH^2 \right\}.
\end{equation}
Using the fact $\mathcal{E}(f)=\int_Z||f(x)-y||_Y^2d\rho(x,y)=||L_K^{1/2}(f-f_\HH)||_\HH^2+\mathcal{E}(f_\HH)$, we get,
\begin{equation}\label{fl}
f_{\la}^s=\left(P_{\xx_s}L_KP_{\xx_s}+\la_0I+\sum\limits_{j=1}^p\la_jP_{\xx_s}B_j^*B_jP_{\xx_s}\right)^{-1}P_{\xx_s}L_KP_{\xx_s} f_\HH.
\end{equation}
We assume
\begin{equation}\label{fl.funl.sgl}
f_{\la_0}^s:=\mathop{\text{arg}\min}_{f \in \HH^{\xx_s}} \left\{\int_Z||f(x)-y||_Y^2d\rho(x,y)+\la_0||f||_\HH^2 \right\}
\end{equation}
which implies
$$f_{\la_0}^s=(P_{\xx_s}L_KP_{\xx_s}+\la_0I)^{-1}P_{\xx_s}L_KP_{\xx_s}f_{\HH},$$
where the integral operator $L_K$ is a self-adjoint, non-negative, compact operator on the Hilbert space $\left(\LL_{\rho_X}^2,\langle\cdot,\cdot\rangle_{\LL^2_{\rho_X}}\right)$ of square-integrable functions from $X$ to $Y$ with respect to $\rho_X$, defined as
$$L_K(f)(x):=\int_X K(x,t)f(t)d\rho_X(t),~~x \in X.$$
The integral operator is bounded by $\kappa^2$. The integral operator $L_K$ can also be defined as a self-adjoint operator on $\HH$. We use the same notation $L_K$ for both the operators defined on different domains. Though it is notational abuse, for convenience we use the same notation $L_K$ for both the operators defined on different domains. It is well-known that $L_K^{1/2}$ is an isometry from the space of square integrable functions to reproducing kernel Hilbert space (For more properties see \cite{CuckerSmale,Cucker}).

Our aim is to discuss the convergence issues of the regularized solution $f_{\zz,\la}^s$ based on Nystr{\"o}m type subsampling. We estimate the error bounds of $f_{\zz,\la}^s-f_\HH$ by measuring the bounds of sample error $f_{\zz,\la}^s-f_{\la}^s$ and approximation error $f_\la^s-f_\HH$. The approximation error is estimated with the help of the single-penalty regularized solution $f_{\la_0}^s$.

For any probability measure, we can always obtain a solution converging to the prescribed target function but the convergence rates may be arbitrarily slow. This phenomena is known as no free lunch theorem \cite{Devroye}. Therefore, we need some prior assumptions on the probability measure $\rho$ in order to achieve the uniform convergence rates for learning algorithms. Following the notion of Bauer et al. \cite{Bauer}, Caponnetto and De Vito \cite{Caponnetto}, we consider the following assumptions on the joint probability measure $\rho$ in terms of the complexity of the target function and a theoretical spectral parameter effective dimension:
\begin{enumerate}[(i)]
\item For the probability measure on $X\times Y$,
\begin{equation}\label{Y.leq.M.1}
\int_Z||y||_Y^2~d\rho(x,y)<\infty
\end{equation}
\item  There exists the minimizer of the generalization error over the RKHS $\HH$,
\begin{equation}\label{target.fun}
f_\HH:=\mathop{\text{arg}\min}_{f \in \HH} \left\{\int_Z||f(x)-y||_Y^2d\rho(x,y)\right\}.
\end{equation}
\item There exist some constants $M,\Sigma$ such that
\begin{equation}\label{Y.leq.M.2}
\int_Y\left(e^{||y-f_\HH(x)||_Y/M}-\frac{||y-f_\HH(x)||_Y}{M}-1\right)d\rho(y|x)\leq\frac{\Sigma^2}{2M^2}
\end{equation}
holds for almost all $x\in X$.
\end{enumerate}

It is worthwhile to observe that for the real-valued functions and multi-task learning algorithms, the boundedness of output space $Y$ can be easily ensured. So we can get the error estimates from our analysis without imposing any condition on the conditional probability measure (\ref{Y.leq.M.2}).

The smoothness of the target function can be described in terms of the integral operator by the source condition:
\begin{assumption}{\bf (Source condition)}\label{source.cond}
Suppose
$$\Omega_{\phi,R}:=\left\{f \in \HH: f=\phi(L_K)g \text{ and }||g||_\HH \leq R\right\},$$
where $\phi$ is operator monotone function on the interval $[0,\kappa^2]$ with the assumption $\phi(0)=0$ and $\phi^2$ is a concave function. Then the condition  $f_\HH\in\Omega_{\phi,R}$ is usually referred to as general source condition \cite{Mathe}.
\end{assumption}

\begin{assumption}{\bf (Polynomial decay condition)}\label{poly.decay}
For fixed positive constants $\alpha,\beta$ and $b>1$, we assume that the eigenvalues $t_n$'s of the integral operator $L_K$ follows the polynomial decay:
$$\alpha n^{-b}\leq t_n\leq\beta n^{-b}~~\forall n\in\NN.$$
\end{assumption}
We define the class of the probability measures $\PP_\phi$ satisfying the conditions (i), (ii), (iii) and Assumption \ref{source.cond}. We also consider the probability measure class $\PP_{\phi,b}$ which satisfies  the conditions (i), (ii), (iii) and Assumption \ref{source.cond}, \ref{poly.decay}.

The convergence rates discussed in our analysis depend on the effective dimension. We achieve the optimal minimax convergence rates using the concept of the effective dimension. For the integral operator $L_K$, the effective dimension is defined as
$$\mathcal{N}(\gamma):=Tr\left((L_K+\gamma I)^{-1}L_K\right), \text{  for }\gamma>0.$$
The fact, $L_K$ is a trace class operator implies that the effective dimension is finite. The effective dimension is continuously decreasing function of $\gamma$ from $\infty$ to $0$. For further discussion on effective dimension we refer to the literature \cite{Blanchard1,Blanchard16,Lin,Lu16,Zhang}.

The effective dimension $\mathcal{N}(\gamma)$ can be estimated from Proposition 3 \cite{Caponnetto} under the polynomial decay condition as follows,
\begin{equation}\label{N(l).bound}
\mathcal{N}(\gamma) \leq \frac{\beta b}{b-1}\gamma^{-1/b},\text{ for }b>1
\end{equation}
and without the polynomial decay condition, we have
$$\mathcal{N}(\gamma)\leq ||(L_K+\gamma I)^{-1}||_{\mathcal{L}(\HH)}Tr\left(L_K\right) \leq \frac{\kappa^2}{\gamma}.$$

We define the random variable $\mathcal{N}_x(\gamma)=\langle K_x,(L_K+\gamma I)^{-1}K_x\rangle_\HH$ for $x\in X$ and let $$\mathcal{N}_\infty(\gamma):=\sup\limits_{x\in X}\mathcal{N}_x(\gamma)<\infty.$$

\begin{table}[h!]
\begin{center}
\begin{tabular}{|m{1.4cm}|l|l|m{2.3cm}|m{2.1cm}|m{1.4cm}|m{1.3cm}|}
\hline
& $||f_{\zz,\la}\hspace{-.1cm}-\hspace{-.1cm}f_\HH||_\rho$ & $||f_{\zz,\la}\hspace{-.1cm}-\hspace{-.1cm}f_\HH||_\HH$ & \small{Assumption ~~~~ ($p$\hspace{-.3mm} qualification)} & \small{  ~~~~Scheme} & \small{general source condition} & \small{Optimal rates} \\
\hline
\small{Kriukova et al. \cite{Kriukova16}} &$m^{-\frac{2r+1}{4r+4}}$ & ~~~~~N/A  & ~~~~~$r\leq \frac{1}{2}$ & \small{Single-penalty Tikhonov regularization} &$~~\surd$&\\
\hline
\small{Rudi et al. \cite{Rudi}} & $m^{-\frac{2br+b}{4br+2b+2}}$ & ~~~~~N/A  & ~~~~~$r\leq \frac{1}{2}$ & \small{Single-penalty Tikhonov regularization} &&$~~\surd$\\
\hline
\small{Our Results} & $m^{-\frac{2br+b}{4br+2b+2}}$ & $m^{-\frac{br}{2br+b+1}}$ & ~~~~~$r\leq \frac{1}{2}$ & \small{Multi-penalty regularization} &$~~\surd$&$~~\surd$\\
\hline
\end{tabular}
\end{center}
\caption{Convergence rates of the regularized learning algorithms based on Nystr{\"o}m subsampling}\label{comparision}
\end{table}

Now we review the previous results on the regularization schemes based on Nystr{\"o}m subsampling which are directly comparable to our results: Kriukova et al. \cite{Kriukova16} and Rudi et al. \cite{Rudi}. For convenience, we tried to present the most essential points in the unified way in Table \ref{comparision}. We have shown the convergence rates under H{\"o}lder's source condition. Rudi et al. \cite{Rudi} obtained the minimax optimal convergence rates depending on the eigenvalues of $L_K$ in $||\cdot||_\rho$-norm. To obtain the optimal rates the concept of effective dimension is exploited. Kriukova et al. \cite{Kriukova16} considered the Tikhonov regularization with Nystr{\"o}m type subsampling under general source condition. They discussed the upper conevergence rates and do not take into account the polynomial decay condition of the eigenvalues of the integral operator $L_K$. We used the idea of Nystr{\"o}m type subsampling to efficiently implement the multi-penalty regularization algorithm. We obtain optimal convergence rates of multi-penalty regularization with Nystr{\"o}m type subsampling under general source condition. In particular, we also get optimal rates of single-penalty Tikhonov regularization with Nystr{\"o}m type subsampling under general source condition as the special case.

\section{Convergence issues}
In this section, we present the optimal minimax convergence rates for vector-valued multi-penalty regularization based on Nystr{\"o}m type subsampling using the concept of effective dimension over the classes of the probability measures $\mathcal{P}_\phi$ and $\mathcal{P}_{\phi,b}$.

In order to prove the optimal convergence rates, we need the following inequality which is used in the papers \cite{Bauer,Caponnetto} and based on the results of Pinelis and Sakhanenko \cite{Pinelis}.
\begin{proposition}\label{pinels_lemma}
Let $\xi$ be a random variable on the probability space $(\Omega,\mathcal{B},P)$ with values in real separable Hilbert space $\HH$. If there exist two constants $Q$ and $S$ satisfying
\begin{equation}\label{pinels_ineq}
E\left\{||\xi-E(\xi)||_{\HH}^n\right\} \leq \frac{1}{2}n!S^2Q^{n-2}~~~\forall n \geq 2,
\end{equation}
then for any $0<\eta<1$ and for all $m \in \NN$,
$$Prob\left\{(\omega_1,\ldots,\omega_m) \in \Omega^m : \left|\left|\frac{1}{m}\sum\limits_{i=1}^m [\xi(\omega_i)-E(\xi(\omega_i))]\right|\right|_{\HH}\leq 2\left(\frac{Q}{m}+\frac{S}{\sqrt{m}}\right)\log\left(\frac{2}{\eta}\right)\right\} \geq 1-\eta.$$
In particular, the inequality (\ref{pinels_ineq}) holds if
$$||\xi(\omega)||_{\HH}\leq Q \text{ and } E(||\xi(\omega)||_\HH^2)\leq S^2.$$
\end{proposition}

In the following proposition, we measure the effect of random sampling using noise assumption (\ref{Y.leq.M.2}) in terms of the effective dimension $\mathcal{N}(\gamma)$. The quantity describes the probabilistic estimates of the perturbation measure due to random sampling.
\begin{proposition}\label{main.bound}
Let $\zz$ be i.i.d. samples drawn according to the probability measure $\rho$ satisfying the assumptions (\ref{Y.leq.M.1}), (\ref{target.fun}), (\ref{Y.leq.M.2}) and $\kappa=\sqrt{\sup\limits_{x\in X}Tr(K_x^*K_x)}$. Then for all $0<\eta<1$, with the confidence $1-\eta$, we have
\begin{equation}\label{LK.I.app}
||(L_K+\gamma I)^{-1/2}P_{\xx_s}\{S_{\xx}^*\yy-S_{\xx}^*S_{\xx}f_\HH\}||_\HH \leq 2\left(\frac{\kappa M}{m\sqrt{\gamma}}+\sqrt{\frac{\Sigma^2\mathcal{N}(\gamma)}{m}}\right)\log\left(\frac{4}{\eta}\right)
\end{equation}
and
\begin{equation}\label{Sx.Sx.LK}
||S_{\xx}^*S_{\xx}-L_K||_{\mathcal{L}(\HH)}\leq 2\left(\frac{\kappa^2}{m}+\frac{\kappa^2}{\sqrt{m}}\right)\log\left(\frac{4}{\eta}\right).
\end{equation}
\end{proposition}
\begin{proof}
To estimate the first expression, we consider the random variable $\xi_1(z)=(L_K+\gamma I)^{-1/2} P_{\xx_s}K_x(y-f_\HH(x))$ from $(Z,\rho)$ to reproducing kernel Hilbert space $\HH$ with
$$E_z(\xi_1)=\int_Z(L_K+\gamma I)^{-1/2}P_{\xx_s}K_x(y-f_\HH(x))d\rho(x,y)=0,$$
$$\frac{1}{m}\sum\limits_{i=1}^m\xi_1(z_i)=(L_K+\gamma I)^{-1/2}P_{\xx_s}(S_{\xx}^*\yy-S_{\xx}^*S_{\xx}f_\HH)$$
and
\begin{eqnarray*}
E_z(||\xi_1-E_z(\xi_1)||_{\HH}^n)&=&E_z\left(||(L_K+\gamma I)^{-1/2}P_{\xx_s}K_x(y-f_\HH(x))||_{\HH}^n\right) \\
&\leq& E_z\left(||K_x^*P_{\xx_s}(L_K+\gamma I)^{-1}P_{\xx_s}K_x||_{\mathcal{L}(Y)}^{n/2}||y-f_\HH(x)||_{Y}^n\right) \\
&\leq& E_x\left(||K_x^*P_{\xx_s}(L_K+\gamma I)^{-1}P_{\xx_s}K_x||_{\mathcal{L}(Y)}^{n/2}E_y\left(||y-f_\HH(x)||_{Y}^n\right)\right).
\end{eqnarray*}
Under the assumption (\ref{Y.leq.M.2}) we get,
$$E_z\left(||\xi_1-E_z(\xi_1)||_{\HH}^n\right) \leq \frac{n!}{2}\left(\Sigma\sqrt{\mathcal{N}(\gamma)}\right)^2\left(\frac{\kappa M}{\sqrt{\gamma}}\right)^{n-2},~~\forall n\geq 2.$$
On applying Proposition \ref{pinels_lemma} we conclude that
\begin{equation*}
||(L_K+\gamma I)^{-1/2}P_{\xx_s}\{S_{\xx}^*\yy-S_{\xx}^*S_{\xx}f_\HH\}||_\HH \leq 2\left(\frac{\kappa M}{m\sqrt{\gamma}}+\sqrt{\frac{\Sigma^2\mathcal{N}(\gamma)}{m}}\right)\log\left(\frac{4}{\eta}\right)
\end{equation*}
with confidence $1-\eta/2$.

The second expression can be estimated easily by considering the random variable $\xi_2(x)=K_xK_x^*$ from $(X,\rho_X)$ to $\mathcal{L}(\HH)$. The proof can also be found in De Vito et al. \cite{DeVito0}.
\end{proof}

The following conditions on the sample size and sub-sample size are used to derive the convergence rates of regularized learning algorithms. In particular, we can assume the following inequality for sufficiently large sample with the confidence $1-\eta$:
\begin{equation}\label{suff_sample}
\frac{8\kappa^2}{\sqrt{m}}\log\left(\frac{4}{\eta}\right) \leq \la_0.
\end{equation}

Following the notion of Rudi et al. \cite{Rudi} and Kriukova et al. \cite{Kriukova16} on subsampling, we measure the approximation power of the projection method induced by the projection operator $P_{\xx_s}$ in terms of $\Delta_s:=||L_K^{1/2}(I-P_{\xx_s})||_{\mathcal{L}(\HH)}$. We make the assumption on $\Delta_s$ as considered in Theorem 2 \cite{Kriukova16}:
\begin{equation}\label{sampling}
\Delta_s=||L_K^{1/2}(I-P_{\xx_s})||_{\mathcal{L}(\HH)}\leq \sqrt{\Theta_{1/2}^{-1}(m^{-1/2})}, \text{    for } \Theta_{1/2}(t)=\sqrt{t}\phi(t).
\end{equation}

Under the parameter choice $\la_0=\Psi^{-1}(m^{-1/2})$ for $\Psi(t)=t^{\frac{1}{2}+\frac{1}{2a}}\phi(t)$ $(a\geq 1)$, we obtain
\begin{equation}\label{sampling.para}
\Delta_s^2\leq \Theta_{1/2}^{-1}(m^{-1/2})\leq \Psi^{-1}(m^{-1/2})=\la_0.
\end{equation}

Moreover, from Lemma 6 \cite{Rudi} under Assumption \ref{poly.decay} and  $s\geq \max\left\{67\log\left(\frac{12\kappa^2}{\la_0\delta}\right), 5\mathcal{N}_\infty\left(\frac{\la_0}{3}\right)\log\left(\frac{12\kappa^2}{\la_0\delta}\right)\right\}$, $\la_0>0$, for every $\delta>0$, the following inequality holds with the probability $1-\delta$,
$$\Delta_s^2\leq \left|\left|\left(L_K+\frac{\la_0}{3}I\right)^{1/2}(I-P_{\xx_s})\right|\right|_{\mathcal{L}(\HH)}^2\leq \la_0.$$
Then under the condition (\ref{sampling.para}) using Proposition 2, 3 \cite{Mathe2003} we get,
$$||(I-P_{\xx_s})\psi(L_K)||_{\mathcal{L}(\HH)}\leq \psi\left(||L_K^{1/2}(I-P_{\xx_s})||_{\mathcal{L}(\HH)}^2\right)\leq \psi(\la_0)$$
and
$$||P_{\xx_s}\psi(L_K)P_{\xx_s}-\psi(P_{\xx_s}L_KP_{\xx_s})||_{\mathcal{L}(\HH)}\leq c_\psi\psi\left(||L_K^{1/2}(I-P_{\xx_s})||_{\mathcal{L}(\HH)}^2\right)\leq c_\psi\psi(\la_0).$$

In the following section, we discuss the error analysis of the multi-penalty regularization scheme based on Nystr{\"o}m type subsampling in probabilistic sense. In general, we derive the convergence rates for regularization algorithms in the norm in RKHS and the norm in $\LL^2_{\rho_X}$ separately. In Theorem \ref{psi.rates}, \ref{psi.rates.P1}, \ref{psi.rates.P2}, we estimate error bounds for multi-penalty regularization based on Nystr{\"o}m type subsampling in $\psi$-weighted norm which consequently provides the convergence rates of the regularized solution $f_{\zz,\la}^s$ in both $||\cdot||_\HH$-norm and $||\cdot||_\rho$-norm.
\begin{theorem}\label{psi.rates}
Let $\zz$ be i.i.d. samples drawn according to the probability measure $\rho\in\mathcal{P}_\phi$ with the assumption that $t^{-1/2}\psi(t)$, $t^{-1}\phi(t)$, $t^{-1}\phi(t)\psi(t)$ are nonincreasing functions. Then under the parameter choice $\la_0=\Psi^{-1}(m^{-1/2})$ for $\Psi(t)=t^{\frac{1}{2}+\frac{1}{2a}}\phi(t)$ $(a\geq 1)$, for sufficiently large sample according to (\ref{suff_sample}) and for subsampling according to (\ref{sampling.para}), the following convergence rates of $f_{\zz,\la}^s$ holds with the confidence $1-\eta$ for all $0<\eta<1$,
$$||\psi(L_K)(f_{\zz,\la}^s-f_\HH)||_\HH\leq \psi(\la_0)\left\{c_1\phi(\la_0)+c_2\frac{\mathcal{B}_\la}{\la_0^{3/2}} +c_3\frac{1}{m\la_0}+c_4\sqrt{\frac{\mathcal{N}(\la_0)}{m\la_0}}\right\}\log\left(\frac{4}{\eta}\right),$$
where $c_1=6R+(5+c_\psi)(3+c_\phi)R$, $c_2=(5+c_\psi)||f_\HH||_\rho$, $c_3=8\kappa M$, $c_4=8\Sigma$ and $\mathcal{B}_\la=||\sum_{j=1}^p\la_jB_j^*B_j||$.
\end{theorem}
\begin{proof}
We discuss the error bound for $||\psi(L_K)(f_{\zz,\la}^s-f_\HH)||_\HH$ by estimating the expressions $||\psi(L_K)(f_{\zz,\la}^s-f_{\la}^s)||_\HH$ and $||\psi(L_K)(f_\la^s-f_\HH)||_\HH$.
The first term can be expressed as
\begin{eqnarray*}
f_{\zz,\la}^s-f_\la^s&=&(P_{\xx_s}S_{\xx}^*S_{\xx}P_{\xx_s}+\la_0I+\sum\limits_{j=1}^p\la_jP_{\xx_s}B_j^*B_jP_{\xx_s})^{-1} \{P_{\xx_s}S_{\xx}^*\yy-P_{\xx_s}S_{\xx}^*S_{\xx}P_{\xx_s}f_\HH\\
&&+(P_{\xx_s}S_{\xx}^*S_{\xx}P_{\xx_s}-P_{\xx_s}L_KP_{\xx_s}) (f_\HH-f_\la^s)\}
\end{eqnarray*}
which implies
$$||\psi(L_K)(f_{\zz,\la}^s-f_\la^s)||_\HH\leq \frac{\psi(\la_0)}{\sqrt{\la_0}}I_1\left\{I_2+\frac{1}{\sqrt{\la_0}}(I_3+I_4||f_\HH-f_{\la}^s||_\HH)\right\},$$
where $I_1=||(L_K+\la_0I)^{1/2}(P_{\xx_s}S_{\xx}^*S_{\xx}P_{\xx_s}+\la_0I +\sum\limits_{j=1}^p\la_jP_{\xx_s}B_j^*B_jP_{\xx_s})^{-1}(L_K+\la_0I)^{1/2} ||_{\mathcal{L}(\HH)}$, $I_2=||(L_K+\la_0I)^{-1/2}(P_{\xx_s}S_{\xx}^*\yy-P_{\xx_s}S_{\xx}^*S_{\xx}f_\HH)||_\HH$, $I_3=||P_{\xx_s}S_{\xx}^*S_{\xx}(I-P_{\xx_s})f_\HH||_\HH$ and $I_4=||S_{\xx}^*S_{\xx}-L_K||_{\mathcal{L}(\HH)}$.

The estimates of $I_2$ and $I_4$ can be obtained from Proposition \ref{main.bound}. Under the condition (\ref{suff_sample}) using the second estimate of Proposition \ref{main.bound}, we obtain
$$Tr\left((P_{\xx_s}L_KP_{\xx_s}+\la_0I)^{-1}(P_{\xx_s}L_KP_{\xx_s}-P_{\xx_s}S_{\xx}^*S_{\xx}P_{\xx_s})\right)\leq \frac{I_4}{\la_0}\leq \frac{4\kappa^2}{\sqrt{m}\la_0}\log\left(\frac{4}{\eta}\right)\leq \frac{1}{2}$$
and under the norm inequalities $||A||_{\mathcal{L}(\HH)}\leq Tr(|A|)$, $Tr(AB)\leq Tr(A)||B||_{\mathcal{L}(\HH)}$ which implies
\begin{eqnarray*}
I_1&\leq&||(L_K+\la_0I)^{1/2}(P_{\xx_s}S_{\xx}^*S_{\xx}P_{\xx_s}+\la_0I)^{-1}(L_K+\la_0I)^{1/2} ||_{\mathcal{L}(\HH)}  \\
&\leq& Tr\left((P_{\xx_s}S_{\xx}^*S_{\xx}P_{\xx_s}+\la_0I)^{-1}(L_K+\la_0I)\right)\\
&=& Tr\left((P_{\xx_s}S_{\xx}^*S_{\xx}P_{\xx_s}+\la_0I)^{-1}((I-P_{\xx_s})L_K+P_{\xx_s}L_K(I-P_{\xx_s}))\right)\\ &&+Tr\left((P_{\xx_s}S_{\xx}^*S_{\xx}P_{\xx_s}+\la_0I)^{-1}(P_{\xx_s}L_KP_{\xx_s}+\la_0I)\right)\\
&\leq&\frac{2}{\la_0}||L_K(I-P_{\xx_s})||_{\mathcal{L}(\HH)}+Tr(\{I-(P_{\xx_s}L_KP_{\xx_s}+\la_0I)^{-1} (P_{\xx_s}L_KP_{\xx_s}-P_{\xx_s}S_{\xx}^*S_{\xx}P_{\xx_s})\}^{-1}) \leq 4.
\end{eqnarray*}
Under the smoothness assumption $f_\HH \in \Omega_{\phi,R}$ there exists $g \in \HH$ such that $f_\HH=\phi(L_K)g$ and $||g||_\HH\leq R$.
\begin{eqnarray*}
I_3=||P_{\xx_s}S_{\xx}^*S_{\xx}(I-P_{\xx_s})f_\HH||_\HH &\leq& R||P_{\xx_s}S_{\xx}^*S_{\xx}(I-P_{\xx_s})||_{\mathcal{L}(\HH)} ||(I-P_{\xx_s})\phi(L_K)||_{\mathcal{L}(\HH)}\\
&\leq& R\phi(||L_K^{1/2}(I-P_{\xx_s})||_{\mathcal{L}(\HH)}^2)(||S_{\xx}^*S_{\xx}-L_K||_{\mathcal{L}(\HH)}+||L_K(I-P_{\xx_s})||_{\mathcal{L}(\HH)}) \\
&\leq&  R\phi(||L_K^{1/2}(I-P_{\xx_s})||_{\mathcal{L}(\HH)}^2)\left(\frac{4\kappa^2}{\sqrt{m}}\log\left(\frac{4}{\eta}\right) +||L_K^{1/2}(I-P_{\xx_s})||_{\mathcal{L}(\HH)}^2\right).
\end{eqnarray*}
Using the conditions (\ref{suff_sample}) and (\ref{sampling.para}), we get
$$I_3\leq \frac{3}{2}R\la_0\phi(\la_0).$$
Therefore,
\begin{equation}\label{1.bound}
||\psi(L_K)(f_{\zz,\la}^s-f_\la^s)||_\HH\leq 2\psi(\la_0)\left\{2I_2/\sqrt{\la_0}+3R\phi(\la_0)+||f_\la^s-f_\HH||_\HH\right\}.
\end{equation}
For the operator monotone function $\psi$, we consider the error term:
\begin{equation*}
||\psi(L_K)(f_{\la}^s-f_\HH)||_{\HH}\leq (I_5+I_6+I_7)||f_{\la}^s-f_\HH||_\HH +||\psi(P_{\xx_s}L_KP_{\xx_s})(f_{\la}^s-f_\HH)||_\HH,
\end{equation*}
where $I_5=||(I-P_{\xx_s})\psi(L_K)||_{\mathcal{L}(\HH)}$, $I_6=||P_{\xx_s}\psi(L_K)-P_{\xx_s}\psi(L_K)P_{\xx_s}||_{\mathcal{L}(\HH)}$ and $I_7=||P_{\xx_s}\psi(L_K)P_{\xx_s}-\psi(P_{\xx_s}L_KP_{\xx_s})||_{\mathcal{L}(\HH)}$.

Hence,
\begin{equation}\label{2.bound}
||\psi(L_K)(f_{\la}^s-f_\HH)||_\HH\leq (2+c_\psi)\psi(\la_0)||f_{\la}^s-f_\HH||_\HH+||\psi(P_{\xx_s}L_KP_{\xx_s})(f_{\la}^s-f_\HH)||_\HH.
\end{equation}
We decompose the term $f_\la^s-f_\HH$ into three parts $f_\la^s-f_{\la_0}^s$, $f_{\la_0}^s-P_{\xx_s}f_\HH$ and $P_{\xx_s}f_\HH-f_\HH$. Then the expression
$$f_\la^s-f_{\la_0}^s=-(P_{\xx_s}L_KP_{\xx_s}+\la_0 I+\sum\limits_{j=1}^p\la_jP_{\xx_s}B_j^*B_jP_{\xx_s})^{-1}\sum\limits_{j=1}^p\la_jP_{\xx_s}B_j^*B_jP_{\xx_s}f_{\la_0}^s$$
implies that
$$||f_{\la}^s-f_{\la_0}^s||_\HH\leq \frac{\mathcal{B}_\la}{\la_0}||f_{\la_0}^s||_\HH\leq \frac{\mathcal{B}_\la}{\la_0^{3/2}}||f_\HH||_\rho$$
and
$$||\psi(P_{\xx_s}L_KP_{\xx_s})(f_\la^s-f_{\la_0}^s)||_\HH\leq\frac{\mathcal{B}_\la}{\sqrt{\la_0}}I_8I_9||f_{\la_0}^s||_\HH \leq\frac{\mathcal{B}_\la\psi(\la_0)}{\la_0^{3/2}}||f_\HH||_\rho,$$
where $\mathcal{B}_\la=\left|\left|\sum\limits_{j=1}^p\la_jB_j^*B_j\right|\right|$,
$I_8=||\psi(P_{\xx_s}L_KP_{\xx_s})(P_{\xx_s}L_KP_{\xx_s}+\la_0I)^{-1/2}||_{\mathcal{L}(\HH)}$
and
$I_9=||(P_{\xx_s}L_KP_{\xx_s}+\la_0I)^{1/2}(P_{\xx_s}L_KP_{\xx_s}+\la_0I +\sum\limits_{j=1}^p\la_jP_{\xx_s}B_j^*B_jP_{\xx_s})^{-1}(P_{\xx_s}L_KP_{\xx_s}+\la_0I)^{1/2} ||_{\mathcal{L}(\HH)}$.

The expression
$$f_{\la_0}^s-P_{\xx_s}f_\HH=\{(P_{\xx_s}L_KP_{\xx_s}+\la_0I)^{-1}P_{\xx_s}L_KP_{\xx_s}-I\}P_{\xx_s}\phi(L_K)g$$
gives that
$$||f_{\la_0}^s-P_{\xx_s}f_\HH||_\HH\leq R(I_{10}+I_{11}+I_{12})$$
and
$$||\psi(P_{\xx_s}L_KP_{\xx_s})(f_{\la_0}^s-P_{\xx_s}f_\HH)||_\HH\leq R(I_{10}I_{13}+I_{11}I_{13}+I_{14}),$$
where $I_{10}=||P_{\xx_s}\phi(L_K)(I-P_{\xx_s})||_{\mathcal{L}(\HH)}$, $I_{11}=||P_{\xx_s}\phi(L_K)P_{\xx_s}-\phi(P_{\xx_s}L_KP_{\xx_s})||_{\mathcal{L}(\HH)}$, $I_{12}=||\{(P_{\xx_s}L_KP_{\xx_s}+\la_0I)^{-1}P_{\xx_s}L_KP_{\xx_s}-I\}\phi(P_{\xx_s}L_KP_{\xx_s})||_{\mathcal{L}(\HH)}$,
$I_{13}=||\psi(P_{\xx_s}L_KP_{\xx_s})\{(P_{\xx_s}L_KP_{\xx_s}+\la_0I)^{-1} P_{\xx_s}L_KP_{\xx_s}-I\}||_{\mathcal{L}(\HH)}$
and
$I_{14}=||\psi(P_{\xx_s}L_KP_{\xx_s})\{(P_{\xx_s}L_KP_{\xx_s}+\la_0I)^{-1} P_{\xx_s}L_KP_{\xx_s}-I\}\phi(P_{\xx_s}L_KP_{\xx_s})||_{\mathcal{L}(\HH)}$.

Again using the conditions on $\psi$ and $\phi$, we get
$$||f_{\la_0}^s-P_{\xx_s}f_\HH||_\HH\leq R(2+c_\phi)\phi(\la_0)$$
and
$$||\psi(P_{\xx_s}L_KP_{\xx_s})(f_{\la_0}^s-P_{\xx_s}f_\HH)||_\HH\leq R(2+c_\phi)\psi(\la_0)\phi(\la_0).$$
We also have,
$$||P_{\xx_s}f_\HH-f_\HH||_\HH\leq R||(I-P_{\xx_s})\phi(L_K)||_{\mathcal{L}(\HH)}\leq R\phi(||L_K^{1/2}(I-P_{\xx_s})||_{\mathcal{L}(\HH)}^2)\leq R\phi(\la_0)$$
and
$$||\psi(L_K)(P_{\xx_s}f_\HH-f_\HH)||_\HH\leq  R\psi(\la_0)\phi(\la_0).$$
Hence we obtain,
\begin{equation}\label{approx.bound}
||f_{\la}^s-f_\HH||_\HH\leq R(3+c_\phi)\phi(\la_0)+\frac{\mathcal{B}_\la}{\la_0^{3/2}}||f_\HH||_\rho
\end{equation}
and
\begin{equation}\label{psi.approx.bound}
||\psi(P_{\xx_s}L_KP_{\xx_s})(f_{\la}^s-f_\HH)||_\HH\leq \psi(\la_0)\left\{R(3+c_\phi)\phi(\la_0)+\frac{\mathcal{B}_\la}{\la_0^{3/2}}||f_\HH||_\rho\right\}.
\end{equation}
Combining the bounds (\ref{approx.bound}), (\ref{psi.approx.bound}) with inequalities (\ref{1.bound}) and (\ref{2.bound}) we obtain the desired result.
\end{proof}

In Theorem \ref{psi.rates}, the error estimates reveal the interesting fact that the error terms consist of increasing and decreasing function of $\al$ which led to propose a posteriori choice of regularization parameter $\al$ based on balancing principle. Hence the effective dimension plays the crucial role in the error analysis of regularized learning algorithms.

Here the upper convergence rates of the regularized solution $f_{\zz,\la}$ are derived from the estimates of Theorem \ref{psi.rates} for the class of probability measure $P_{\phi}$, $P_{\phi,b}$, respectively. In Theorem \ref{psi.rates.P1}, we discuss the error estimates under the general source condition and the parameter choice rule based on the index function $\phi$ and sample size $m$. Under the polynomial decay condition, in Theorem \ref{psi.rates.P2} we obtain the optimal minimax convergence rates in terms of index function $\phi$, the parameter $b$ and the number of samples $m$.

\begin{theorem}\label{psi.rates.P1}
Under the same assumptions of Theorem \ref{psi.rates} with the parameter choice $\la_0\in(0,1]$, $\la_0=\Theta^{-1}(m^{-1/2})$, $\la_j=(\Theta^{-1}(m^{-1/2}))^{3/2}\phi(\Theta^{-1}(m^{-1/2}))$ for $1\leq j\leq p$, where $\Theta(t)=t\phi(t)$, the convergence rates of $f_{\zz,\la}^s$ can be described as follows:
$$\underset{\zz\in Z^m}{\text{Prob}}\left\{||\psi(L_K)(f_{\zz,\la}^s-f_\HH)||_\HH  \leq C\psi(\Theta^{-1}(m^{-1/2}))\phi(\Theta^{-1}(m^{-1/2}))\log\left(\frac{4}{\eta}\right)\right\}\geq 1-\eta.$$
\end{theorem}
\begin{proof}
Let $\Theta(t)=t\phi(t)$. Then it follows,
$$\lim\limits_{t \to 0}\frac{\Theta(t)}{\sqrt{t}}=\lim\limits_{t \to 0}\frac{t^2}{\Theta^{-1}(t)}=0.$$
Under the parameter choice $\la_0=\Theta^{-1}(m^{-1/2})$ we have,
$$\lim\limits_{m\to\infty}m\la_0=\infty.$$
Therefore for sufficiently large $m$, we get $m\la_0\geq 1$ and
$$\frac{1}{m\la_0}=\frac{\la_0^{1/2}\phi(\la_0)}{\sqrt{m\la_0}}\leq \la_0^{1/2}\phi(\la_0).$$
Under the parameter choice $\la_0 \leq 1$, $\la_0=\Theta^{-1}(m^{-1/2})$, $\la_j=(\Theta^{-1}(m^{-1/2}))^{3/2}\phi(\Theta^{-1}(m^{-1/2}))$ for $1\leq j\leq p$, from Theorem \ref{psi.rates} follows that with the confidence $1-\eta$,
\begin{equation*}
||\psi(L_K)(f_{\zz,\la}^s-f_\HH)||_\HH \leq C\psi(\Theta^{-1}(m^{-1/2}))\phi(\Theta^{-1}(m^{-1/2}))\log\left(\frac{4}{\eta}\right).
\end{equation*}
Hence our conclusion follows.
\end{proof}

\begin{theorem}\label{psi.rates.P2}
Under the same assumptions of Theorem \ref{psi.rates} and Assumption \ref{poly.decay} with the parameter choice $\la_0\in(0,1]$, $\la_0=\Psi^{-1}(m^{-1/2})$, $\la_j=(\Psi^{-1}(m^{-1/2}))^{3/2}\phi(\Psi^{-1}(m^{-1/2}))$ for $1\leq j\leq p$, where $\Psi(t)=t^{\frac{1}{2}+\frac{1}{2b}}\phi(t)$, the convergence rates of $f_{\zz,\la}^s$ can be described as follows:
$$\underset{\zz\in Z^m}{\text{Prob}}\left\{||\psi(L_K)(f_{\zz,\la}^s-f_\HH)||_\HH  \leq C'\psi(\Psi^{-1}(m^{-1/2}))\phi(\Psi^{-1}(m^{-1/2}))\log\left(\frac{4}{\eta}\right)\right\}\geq 1-\eta,$$
where $C'=R'(3R+4\kappa M+4\sqrt{\beta b\Sigma^2/(b-1)}+3\sum\limits_{j=1}^p||B_j^*B_j||~||f_\HH||_\rho)$.
\end{theorem}
\begin{proof} Let $\Psi(t)=t^{\frac{1}{2}+\frac{1}{2b}}\phi(t)$. Then it follows,
$$\lim\limits_{t \to 0}\frac{\Psi(t)}{\sqrt{t}}=\lim\limits_{t \to 0}\frac{t^2}{\Psi^{-1}(t)}=0.$$
Under the parameter choice $\la_0=\Psi^{-1}(m^{-1/2})$ we have,
$$\lim\limits_{m\to\infty}m\la_0=\infty.$$
Therefore for sufficiently large $m$, we get $m\la_0\geq 1$ and
$$\frac{1}{m\la_0}=\frac{\la_0^{\frac{1}{2b}}\phi(\la_0)}{\sqrt{m\la_0}}\leq \la_0^{\frac{1}{2b}}\phi(\la_0).$$
Under the parameter choice $\la_0 \leq 1$, $\la_0=\Psi^{-1}(m^{-1/2})$, $\la_j=(\Psi^{-1}(m^{-1/2}))^{3/2}\phi(\Psi^{-1}(m^{-1/2}))$ for $1\leq j\leq p$, from Theorem \ref{psi.rates} and eqn. (\ref{N(l).bound}) follows that with the confidence $1-\eta$,
\begin{equation}\label{fzl.fl.inter}
||\psi(L_K)(f_{\zz,\la}^s-f_\HH)||_\HH\leq C'\psi(\Psi^{-1}(m^{-1/2}))\phi(\Psi^{-1}(m^{-1/2}))\log\left(\frac{4}{\eta}\right).
\end{equation}
Hence our conclusion follows.
\end{proof}

In Theorem \ref{multi.err.upper.bound.p.para.phi}, \ref{multi.err.upper.bound.p.para}, we present the convergence rates of the multi-penalty estimator $f_{\zz,\la}^s$ for the classes of probability measures $\mathcal{P}_\phi$ and $\mathcal{P}_{\phi,b}$ in both RKHS-norm and $\LL^2$-norm. For $\psi(t)=t^\al$, we can also obtain the convergence rates of the regularized solution $f_{\zz,\la}^s$ in the interpolation norm for the parameter $\al\in[0,\frac{1}{2}]$. In particular, we obtain the error estimates in $||\cdot||_\HH$-norm
for $\al=0$ and in $||\cdot||_{L^2_{\rho_X}}$-norm for $\al=\frac{1}{2}$.

\begin{theorem}\label{multi.err.upper.bound.p.para.phi}
Let $\zz$ be i.i.d. samples drawn according to the probability measure $\rho\in\PP_{\phi}$. Then for sufficiently large sample size $m$ according to (\ref{suff_sample}) and for subsampling according to (\ref{sampling.para}) under the parameter choice $\la_0\in(0,1],~\la_0=\Theta^{-1}(m^{-1/2}),~\la_j=(\Theta^{-1}(m^{-1/2}))^{3/2}\phi(\Theta^{-1}(m^{-1/2}))$ for $1\leq j\leq p$, where $\Theta(t)=t\phi(t)$, for all $0<\eta<1$, the following error estimates hold with confidence $1-\eta$,
\begin{enumerate}[(i)]
\item  If $\phi(t)$ and $t/\phi(t)$ are nondecreasing functions. Then we have,
$$\underset{\zz\in Z^m}{\text{Prob}}\left\{||f_{\zz,\la}^s-f_\HH||_\HH \leq C\phi(\Theta^{-1}(m^{-1/2}))\log\left(\frac{4}{\eta}\right)\right\}\geq 1-\eta$$
and
$$\lim\limits_{\tau\rightarrow\infty}\limsup\limits_{m\rightarrow\infty}\sup\limits_{\rho\in \PP_{\phi}} \underset{\zz\in Z^m}{\text{Prob}}\left\{||f_{\zz,\la}^s-f_\HH||_{\HH}>\tau\phi(\Theta^{-1}(m^{-1/2}))\right\}=0.$$
\item  If $\phi(t)$ and $\sqrt{t}/\phi(t)$ are nondecreasing functions. Then we have,
$$\underset{\zz\in Z^m}{\text{Prob}}\left\{||f_{\zz,\la}^s-f_\HH||_\rho\leq C(\Theta^{-1}(m^{-1/2}))^{1/2}\phi(\Theta^{-1}(m^{-1/2}))\log\left(\frac{4}{\eta}\right)\right\}\geq 1-\eta$$
and
$$\lim\limits_{\tau\rightarrow\infty}\limsup\limits_{m\rightarrow\infty}\sup\limits_{\rho\in \PP_{\phi}} \underset{\zz\in Z^m}{\text{Prob}}\left\{||f_{\zz,\la}^s-f_\HH||_\rho>\tau(\Theta^{-1}(m^{-1/2}))^{1/2}\phi(\Theta^{-1}(m^{-1/2}))\right\}=0.$$
\end{enumerate}
\end{theorem}

\begin{theorem}\label{multi.err.upper.bound.p.para}
Let $\zz$ be i.i.d. samples drawn according to the probability measure $\rho\in\PP_{\phi,b}$. Then for sufficiently large sample size $m$ according to (\ref{suff_sample}) and for subsampling according to (\ref{sampling.para}) under the parameter choice $\la_0\in(0,1],~\la_0=\Psi^{-1}(m^{-1/2}),~\la_j=(\Psi^{-1}(m^{-1/2}))^{3/2}\phi(\Psi^{-1}(m^{-1/2}))$ for $1\leq j\leq p$, where $\Psi(t)=t^{\frac{1}{2}+\frac{1}{2b}}\phi(t)$, for all $0<\eta<1$, the following error estimates hold with confidence $1-\eta$,
\begin{enumerate}[(i)]
\item  If $\phi(t)$ and $t/\phi(t)$ are nondecreasing functions. Then we have,
$$\underset{\zz\in Z^m}{\text{Prob}}\left\{||f_{\zz,\la}^s-f_\HH||_\HH \leq C'\phi(\Psi^{-1}(m^{-1/2}))\log\left(\frac{4}{\eta}\right)\right\}\geq 1-\eta$$
and
$$\lim\limits_{\tau\rightarrow\infty}\limsup\limits_{m\rightarrow\infty}\sup\limits_{\rho\in \PP_{\phi,b}} \underset{\zz\in Z^m}{\text{Prob}}\left\{||f_{\zz,\la}^s-f_\HH||_{\HH}>\tau\phi(\Psi^{-1}(m^{-1/2}))\right\}=0.$$
\item  If $\phi(t)$ and $\sqrt{t}/\phi(t)$ are nondecreasing functions. Then we have,
$$\underset{\zz\in Z^m}{\text{Prob}}\left\{||f_{\zz,\la}^s-f_\HH||_\rho\leq C'(\Psi^{-1}(m^{-1/2}))^{1/2}\phi(\Psi^{-1}(m^{-1/2}))\log\left(\frac{4}{\eta}\right)\right\}\geq 1-\eta$$
and
$$\lim\limits_{\tau\rightarrow\infty}\limsup\limits_{m\rightarrow\infty}\sup\limits_{\rho\in \PP_{\phi,b}} \underset{\zz\in Z^m}{\text{Prob}}\left\{||f_{\zz,\la}^s-f_\HH||_\rho>\tau(\Psi^{-1}(m^{-1/2}))^{1/2}\phi(\Psi^{-1}(m^{-1/2}))\right\}=0.$$
\end{enumerate}
\end{theorem}

The lower convergence rates are discussed for any learning algorithm over the class of the probability measures $\mathcal{P}_{\phi,b}$ in Theorem 3.10, 3.12 \cite{Rastogi}. Here we study the upper convergence rates for multi-penalty regularization based on Nystr{\"o}m type subsampling in vector-valued function setting. If the upper convergence rate for the parameter choice $\la=\la(m)$ coincides with the lower convergence rates, then the parameter choice $\la=\la(m)$ is said to be optimal. For the parameter choice $\la=(\la_0,\ldots,\la_p)$, $\la_0=\Psi^{-1}(m^{-1/2})$, $\la_j=(\Psi^{-1}(m^{-1/2}))^{3/2}\phi(\Psi^{-1}(m^{-1/2}))$, $1\leq j\leq p$, Theorem \ref{multi.err.upper.bound.p.para} share the upper convergence rates with the lower minimax rates of Theorem 3.10, 3.12 \cite{Rastogi}. Therefore the choice of the parameter is optimal.

\begin{remark}
Under H\"{o}lder source condition ($\phi(t)=t^r$), we get the order of convergence $\mathcal{O}(m^{-\frac{r}{2r+2}})$ $(\text{for } 0\leq r \leq 1)$ and $\mathcal{O}(m^{-\frac{2r+1}{4r+4}})$ $(\text{for }0\leq r \leq \frac{1}{2})$ for the class of probability measures $\mathcal{P}_{\phi}$ in $||\cdot||_\HH$-norm and $||\cdot||_{\LL^2_{\rho_X}}$-norm, respectively. Also, for the class of probability measures $\mathcal{P}_{\phi,b}$, we obtain the order of convergence $\mathcal{O}(m^{-\frac{br}{2br+b+1}})$ $(\text{for }0\leq r \leq 1)$ and $\mathcal{O}(m^{-\frac{2br+b}{4br+2b+2}})$ $(\text{for }0\leq r \leq \frac{1}{2})$ in $||\cdot||_\HH$-norm and $||\cdot||_{\LL^2_{\rho_X}}$-norm, respectively.
\end{remark}

Therefore, the randomized subsampling methods can break the memory barriers of standard kernel methods, while preserving optimal learning guarantees.

\section{An aggregation approach for Nystr{\"o}m approximants}\label{theoretical.bound}
The size of sub-sample is described in terms of the integral operator and behavior of its eigenvalues values or the regularity of the target function and the approximation power of the projection operator. In practice, it may be difficult to obtain the appropriate size of the sub-sample. In order to overcome this problem, we discuss an aggregation approach based on linear functional strategy. We construct various Nystr{\"o}m approximants corresponding to different subsampling size. Then the strategy tries to accumulate the information hidden inside various approximants to produce the best estimator of the target function. The approach is widely considered in ill-posed inverse problems \cite{Anderssen,Bauer2,JChen,Goldenshluger,Mathe0} as well as learning theory framework \cite{Kriukova,Kriukova16,Pereverzyev,Rastogi17} to aggregate the various regularized solutions. In linear functional strategy, we consider the linear combination of the approximants and try to figure out the combination which is closure to the target function $f_\HH$. For a finite set of Nystr{\"o}m approximants $\{f_{\zz,\la}^{s_i}\in\HH:1 \leq i \leq l\}$, the aggregation approach can be described as:
\begin{equation}\label{lfs_min}
\mathop{\text{arg}\min}_{(c_1,\ldots,c_l) \in \RR^l}\left|\left|\sum\limits_{i=1}^lc_if_{\zz,\la}^{s_i}-f_\HH\right|\right|_\rho.
\end{equation}
The problem of minimization (\ref{lfs_min}) is equivalent to the problem of finding $\textbf{c}=(c_1,\ldots,c_l)$,
$$H\textbf{c}=h,$$
where $H=(\langle f_{\zz,\la}^{s_i},f_{\zz,\la}^{s_j}\rangle_\rho)_{i,j=1}^l$ and $h=(\langle f_\HH,f_{\zz,\la}^{s_i}\rangle_\rho)_{i=1}^l$.

Due to the involvement of the unknown probability distribution of $\rho$ we cannot determine $H$ and $h$ directly. Therefore we approximate $H$ and $h$ by the quantities $\bar{H}=\left(\frac{1}{n}\sum\limits_{r=1}^nf_{\zz,\la}^{s_i}(x_r)f_{\zz,\la}^{s_j}(x_r)\right)_{i,j=1}^l$ and $\bar{h}=\left(\frac{1}{m}\sum\limits_{r=1}^my_rf_{\zz,\la}^{s_i}(x_r)\right)_{i=1}^l$, respectively.
Now the combination vector $\bar{\textbf{c}}$ is given by $\bar{\textbf{c}}=\bar{H}^{-1}\bar{h}$ and the aggregated solution $f_{\zz}=\sum\limits_{i=1}^l\bar{c}_i f_{\zz,\la}^{s_i},~\bar{\textbf{c}}=(\bar{c}_i)_{i=1}^l$ based on linear functional strategy in $\LL^2_{\rho_X}$-norm (LFS-$\LL^2_{\rho_X}$). The convergence rate of the constructed solution $f_\zz$ can be described as

\begin{theorem}\label{lfs.rates}
Let $\zz$ be i.i.d. samples drawn according to the probability measure $\rho$ with the hypothesis (\ref{Y.leq.M.2}). Then for sufficiently large sample size $m$ according to (\ref{suff_sample}) with the confidence $1-\eta$, we have
$$||f_{\zz}-f_\HH||_\rho = \min\limits_{c \in \RR^l}\left|\left|\sum\limits_{i=1}^lc_if_{\zz,\la}^{s_i}-f_\HH\right|\right|_\rho +\mathcal{O}\left(m^{-1/2}\log\left(\frac{4}{\eta}\right)\right)$$
holds with the probability $1-\eta$.
\end{theorem}

The proof of the theorem follows the same steps as of Theorem 10 \cite{Kriukova}. Here we observe that the individual Nystr{\"o}m approximants can be obtained for the particular values of combination vector. The error term in the Theorem \ref{lfs.rates} goes to $0$ for the large sample size. Moreover, the order of the error in Theorem \ref{lfs.rates} is less than the order of error of Nystr{\"o}m approximants in Theorem \ref{multi.err.upper.bound.p.para}. Therefore the error of aggregated solution in Theorem \ref{lfs.rates} is negligible.
\section{Numerical realization}
In our experiments, we demonstrate the performance of multi-penalty regularization based on Nyst{\"o}m type subsampling using the linear functional strategy for both scalar-valued functions and multi-task learning problems. We present the extensive empirical analysis of the multi-view manifold regularization scheme based on Nystr{\"o}m type subsampling for the challenging multi-class image classification and species recognition with attributes. Then we consider the NSL-KDD benchmark data set from UCI machine learning repository for the intrusion detection problem based on Nystr{\"o}m type subsampling.

\subsection{Caltech-101 data set}
We consider the multi-view manifold regularization scheme discussed in \cite{Minh} for multi-class classification on Caltech-101 data set. The regularized solution is constructed based on the different views of the input data. For the given data set, the views are the different features extracted from the input examples. Let $f=(f^1,\ldots,f^v)\in \HH=\HH_{K^1}\times\ldots\times\HH_{K^v}$ be the function associated to the $v$-views of the inputs. We define the combination operator $C:Y^v\to Y$ by $Cf(x)=\sum\limits_{i=1}^v c_if^i(x), \text{  for  } \textbf{c}=(c_1,\ldots,c_v)\in \RR^v$. We consider the multi-view manifold regularization for semi-supervised problem corresponding to the labeled data $\{(x_i,y_i)\}^m_{i=1}$ and unlabeled data $\{x_i\}_{i=m+1}^n$:
\begin{equation}\label{mv_mls}
f_{\zz,\la}=\mathop{\text{arg}\min}_{f \in \HH,\textbf{c}\in S_\alpha^{v-1}} \frac{1}{m}\sum\limits_{i=1}^m ||y_i-Cf(x_i)||_Y^2+\la_A||f||_{\HH}^2+\la_B||(S_{\xx'}^*M_BS_{\xx'})^{1/2}f||_{\HH}^2+\la_W||(S_{\xx'}^*M_WS_{\xx'})^{1/2}f||_{\HH}^2,
\end{equation}
where the regularization parameters $\la_A,\la_B,\la_W\geq 0$, $\xx'=(x_i)_{i=1}^n$, $M_B=I_n\otimes(M_v\otimes I_Y)$ and $M_W=L\otimes I_Y$ are symmetric, positive operators. Here $M_v=vI_v-\textbf{1}_v\textbf{1}_v^T$, $I_v$ is identity of size $v\times v$, $\textbf{1}_v$ is a vector of size $v\times 1$ with ones, $\otimes$ is the Kronecker product and $L$ is a graph Laplacian. The graph Laplacian $L$ is the block matrix of size $n\times n$, with block $(i,j)$ being the $v\times v$ diagonal matrix, given by
\begin{equation}\label{Laplacian}
L_{ij}=\text{diag}(L_{ij}^1,\ldots,L_{ij}^v),
\end{equation}
where the scalar graph Laplacian $L^i$ is induced by the symmetric, nonnegative weight matrix $W^i$.

The first term controls the complexity of the function in the ambient space, the second term between-view regularization which measures the consistency of the component functions across different views and the third term within-view regularization which measures the smoothness of the component functions in their corresponding views.

Now we consider the Nystr{\"o}m type subsampling on the multi-view learning problem. We restrict the minimization problem (\ref{mv_mls}) over the space:
\begin{equation*}
\HH^{\xx_s}:=\{f|f=\sum\limits_{i=1}^sK_{x_i}y_i,~\yy_s=(y_1,\ldots,y_s)\in Y^s\},
\end{equation*}
where $s\ll n$ and $\xx_s=(x_1,\ldots, x_s)$ is a subset of the input points in the training set and $K$ is the kernel corresponding to the RKHS $\HH=\HH_{K^1}\times\ldots\times\HH_{K^v}$. We construct Nystr{\"o}m approximants by minimizing the regularization problem (\ref{mv_mls}) over the space $\HH^{\xx_s}$.

We have to minimize the multi-view manifold regularization problem simultaneously for $f\in\HH^{\xx_s}$ and $\textbf{c}\in\RR^v$. So we first choose a weight vector $\textbf{c}$ on the sphere $S_{\alpha}^{v-1}=\{x\in \RR^v:||x||=\alpha\}$ and minimize for the function $f$ over $\HH^{\xx_s}$. Then we fix the estimator $f$ and optimize for the weight vector $\textbf{c}$. We continue the iterative process until we get the desired accuracy. The regularized solution of the scheme (\ref{mv_mls}) is constructed according to the Theorem 10 \cite{Minh}.

\begin{algorithm}[H]
\caption{Semi-supervised least-square regression and classification based on Nystr{\"o}m subsampling using the aggregation approach}\label{Algo.LFS}
\begin{algorithmic}
\State {\it This algorithm computes the multi-view learning estimators $f_{\zz,\la}^{s_r}$ corresponding to the views $(v_1,\ldots,v_q)$ for the subsampling $s_r$ $(1\leq r \leq l)$. Then the algorithm computes the aggregated solution $f_\zz$ from the regularized solutions $f_{\zz,\la}^{s_r}$ $(1\leq r\leq l)$ based on the linear functional strategy in $\LL^2_{\rho_X}$.}\\
{\bf Input:}
\vspace{-.2cm}
\begin{itemize}
\setlength\itemsep{-.2em}
\item[-] Training data $\{(x_i,y_i)\}_{i=1}^m\bigcup\{x_i\}_{i=m+1}^n$, with $m$ labeled and $n-m$ unlabeled examples.
\item[-] Testing data $t_i$.
\end{itemize}
\vspace{-.1cm}
{\bf Parameters:}
\vspace{-.2cm}
\begin{itemize}
\setlength\itemsep{-.2em}
\item[-] The regularization parameters $\la_A$, $\la_B$, $\la_W$.
\item[-] The weight vector $\textbf{c}$.
\item[-] The number of classes $P$.
\item[-] Scalar-valued kernels $K^i$ corresponding to $i$-th view.
\end{itemize}
\vspace{-.1cm}
{\bf Procedure:}
\vspace{-.2cm}
\begin{itemize}
\setlength\itemsep{-.2em}
\item[-] To calculate the set of estimators $f_{\zz,\la}^{s_r}$ $(1\leq r \leq l)$, compute kernel matrices $G_r[\xx',\xx^s]=\left(G_r(x_i,x_j)\right)_{ij}$ for $1\leq i \leq n$ and $1\leq j\leq s$ corresponding to views $(v_1^r,\ldots,v_q^r)$ according to (\ref{Kernel}).
\item[-] Compute graph Laplacian $L$ according to (\ref{Laplacian}).
\item[-] Compute $B_r=\left((J_m^n\otimes \mathbf{c}\mathbf{c}^T)+m\la_B(I_n\otimes M_v)+m\la_WL\right)G_r[\xx',\xx^s]$, where $J_m^n:\RR^n\to \RR^n$ is a diagonal matrix of size $n\times n$, with the first $m$ entries on the main diagonal being $1$ and the rest being $0$.
\item[-] Compute $C=\mathbf{c}^T\otimes I_P$ and $\mathbf{C}^*=I_{n\times m}\otimes C^*$ for $I_{n\times m}=[I_m,0_{m\times (n-m)}]^T$.
\item[-] Compute $Y_C$ such that $\textbf{C}^*\yy=\text{vec}(Y_C^T)$, where the vectorization of an $n\times P$ matrix $A$, denoted $\text{vec}(A)$, is the $nP\times 1$ column vector obtained by stacking the columns of the matrix $A$ on top of one another.
\item[-] Solve matrix equations $B_rA_r+m\la_AA_r=Y_C$ for $A_r$ $(1\leq r \leq l)$.
\item[-] Compute kernel matrices $G_r[\xx,\xx']=\left(G_r(x_i,x_j)\right)_{ij}$ for $1\leq i \leq m$ and $1\leq j\leq n$.
\item[-] Compute $\bar{H}_{rq}=\frac{1}{n}\text{vec}\left(A_r^TG_r[\xx']^T(I_{n\times m}\otimes \mathbf{c})\right)^T\text{vec}\left(A_q^TG_q[\xx']^T(I_{n\times m}\otimes \mathbf{c})\right)$ and  $\bar{h}_r=\frac{1}{m}\yy^T\text{vec}\left(A_r^TG_r[\xx,\xx']^T(I_{m}\otimes \mathbf{c})\right)$ for $1\leq r,q\leq l$.
\item[-] Compute $\bar{c}=(\bar{c}_1,\ldots,\bar{c}_l)^T=\bar{H}^{-1}\bar{h}$ for $\bar{H}=(\bar{H}_{rq})_{r,q=1}^l$ and $\bar{h}=(\bar{h}_r)_{r=1}^l$.
\item[-] Compute kernel matrices $G_r[t_i,\xx]$ between $t_i$ and $\xx$.
\item[-] Compute the value of estimator $f_\zz$ based on aggregation approach at $t_i$, i.e.,
\vspace{-.4cm}
$$f_{\zz}(t_i)=\sum\limits_{r=1}^l \bar{c}_r f_{\zz,\la}^{K_r}(t_i) =\sum\limits_{r=1}^l \bar{c}_r \text{vec}(A_r^TG_r[t_i,\xx]^T).$$
\end{itemize}
\vspace{-.2cm}
{\bf Output:} Multi-class classification: return index of $\max(C f_\zz(t_i))$.\\
\end{algorithmic}
\end{algorithm}

\begin{table}[h!]
\begin{center}
\begin{tabular}{l|lll}
  \hline
  Estimators & Level-1 views & Level-2 views & Level-3 views \\
  \hline
  MVL-NS (s=2) & 41.31\% & 42.55\% & 42.55\% \\
  MVL-NS (s=5) & 57.71\% & 62.90\% & 63.99\% \\
  LFS-MVL-NS   & 57.80\% & 62.85\% & 64.27\% \\
  MVL-LS-SP    & 57.32\% & 62.11\% & 63.68\% \\
  MVL-LS       & 57.67\% & 63.09\% & 64.20\% \\
  \hline
\end{tabular}
\caption{\small Performance of multi-view learning estimators based on Nystr{\"o}m subsampling and the aggregated solution based on LFS-$\LL^2_{\rho_X}$ using $5$ labeled and $5$ unlabeled images per class from Caltech-101 data set.}\label{Table.MVL.5.5}
\end{center}
\end{table}

\begin{table}[h!]
\begin{center}
\begin{tabular}{l|lll}
  \hline
  Estimators & Level-1 views & Level-2 views & Level-3 views \\
  \hline
  MVL-NS (s=3) & 53.31\% & 56.10\% & 56.82\% \\
  MVL-NS (s=10)& 64.42\% & 69.74\% & 71.09\% \\
  LFS-MVL-NS   & 64.64\% & 69.87\% & 71.09\% \\
  MVL-LS-SP    & 64.51\% & 69.30\% & 70.68\% \\
  MVL-LS       & 64.36\% & 69.54\% & 70.98\% \\
  \hline
\end{tabular}
\caption{\small Performance of multi-view learning estimators based on Nystr{\"o}m subsampling and the aggregated solution based on LFS-$\LL^2_{\rho_X}$ using $10$ labeled and $5$ unlabeled images per class from Caltech-101 data set.}\label{Table.MVL.10.5}
\end{center}
\end{table}

\begin{table}[h!]
\begin{center}
\begin{tabular}{l|lll}
  \hline
  Estimators & Level-1 views & Level-2 views & Level-3 views \\
  \hline
  MVL-NS-optC (s=2) & 35.90\% & 35.64\% & 34.90\% \\
  MVL-NS-optC (s=5) & 61.72\% & 64.29\% & 64.77\% \\
  LFS-MVL-NS-optC   & 61.81\% & 64.40\% & 64.90\% \\
  MVL-LS-optC       & 60.33\% & 64.62\% & 65.16\% \\
  \hline
\end{tabular}
\caption{\small Performance of multi-view learning estimators based on Nystr{\"o}m subsampling and the aggregated solution based on LFS-$\LL^2_{\rho_X}$ with optimal combination operator using $5$ labeled and $5$ unlabeled images per class from Caltech-101 data set.}\label{Table.MVL.5.5.optC}
\end{center}
\end{table}

\begin{table}[h!]
\begin{center}
\begin{tabular}{l|lll}
  \hline
  Estimators & Level-1 views & Level-2 views & Level-3 views \\
  \hline
  MVL-NS-optC (s=3) & 45.93\% & 47.23\% & 48.54\% \\
  MVL-NS-optC (s=10)& 67.06\% & 70.09\% & 71.57\% \\
  LFS-MVL-NS-optC   & 67.15\% & 70.24\% & 71.61\% \\
  MVL-LS-optC       & 65.90\% & 70.39\% & 71.20\% \\
  \hline
\end{tabular}
\caption{\small Performance of multi-view learning estimators based on Nystr{\"o}m subsampling and the aggregated solution based on LFS-$\LL^2_{\rho_X}$ with optimal combination operator using $10$ labeled and $5$ unlabeled images per class from Caltech-101 data set.}\label{Table.MVL.10.5.optC}
\end{center}
\end{table}

We demonstrate the performance of multi-view manifold regularization problem (\ref{mv_mls}) based on Nystr{\"o}m type subsampling on Caltech-101 data set for the challenging task of multi-class image classification. Caltech-101 data set is used for object recognition problem which is provided in Fei-Fei et al. \cite{Fei}. The data set contains $P=102$ classes of objects and each class having 40 to 800 images. We have chosen 15 images randomly from each class. PHOW gray, PHOW color, geometric blurred and self-symmetry are the four features considered by Vedaldi et al. \cite{Vedaldi} which are extracted on three levels of the spatial pyramid. We use the $\chi^2$-kernels corresponding to each view $x^i$ of the input $x=(x^1,\ldots,x^v)$ provided in \cite{Vedaldi}. We define the operator-valued kernel for the manifold learning scheme (\ref{mv_mls}):
\vspace{-.4cm}
\begin{equation}\label{Kernel}
K(x,t)=G(x,t)\otimes I_P\text{  for } G(x,t)=\sum\limits_{i=1}^vK^i(x^i,t^i)e_ie_i^T,
\vspace{-.4cm}
\end{equation}
where $K^i$ is scalar-valued kernel defined on $i$-th view and $e_i=(0,\ldots,1,\ldots,0)\in \RR^v$ is the $i$-th coordinate vector. Three splits of the given data are considered and the results are reported in terms of accuracy by averaging over all three splits. The test set contains 15 objects per class in all experiments. We set $y=(-1,\ldots,1,\ldots,-1)\in Y=\RR^P$, i.e., 1 on the $k$-th place otherwise $-1$ if $x$ belongs to the $k$-th class for the output values.

Here we take the number of labeled data per class $l = \{5, 10\}$ and the number of unlabeled data per class $u=5$. In the results of Table \ref{Table.MVL.5.5}, \ref{Table.MVL.10.5}, \ref{Table.MVL.5.5.optC} $\&$ \ref{Table.MVL.10.5.optC}, we use the same choice of regularization parameters $\la_A=10^{-5}$, $\la_B=10^{-6}$, $\la_W=10^{-6}$ as considered in Minh et al. \cite{Minh}. We have chosen the uniform combination operator $\textbf{c}=\frac{1}{v}(1,\ldots,1)^T\in\RR^v$ in the results of Table \ref{Table.MVL.5.5} $\&$ \ref{Table.MVL.10.5}. MVL-NS represents the multi-view estimator based on Nystr{\"o}m subsampling with subsampling size $s$. Level-1 views represents the upper level of views, i.e., PHOW color and gray $\ell0$, SSIM $\ell0$, GB. Similarly, Level-2 views and Level-3 views represent the middle and lower level of feature's spatial pyramid. LFS-MVL-NS is the aggregated estimator based on LFS-$\LL^2_{\rho_X}$ using Nystr{\"o}m approximants (MVL-NS). MVL-SP is the single-penalty multi-view learning estimator corresponding to the problem (\ref{mv_mls}), i.e. for $\la_B=0$ \& $\la_W=0$.

We present the performance of multi-view estimators based on Nystr{\"o}m type subsampling using the feature on different levels of spatial pyramid. In Table \ref{Table.MVL.5.5}, \ref{Table.MVL.10.5}, we observe that on the aggregating the Nystr{\"o}m approximants based on linear functional strategy we are able to achieve the accuracy of multi-view learning estimator (MVL-LS) of the problem (\ref{mv_mls}). We also show the accuracy of the single-penalty regularizer (MVL-LS-SP). In Table \ref{Table.MVL.5.5.optC}, \ref{Table.MVL.10.5.optC}, we go one step further by choosing the optimal combination operator $\mathbf{c}$. We optimize the combination operator over the sphere $S_\alpha^{v-1}$ for the fixed function $f$. To obtain the optimal weight vector $\textbf{c}$, we created a validation set by selecting 5 examples for each class from the training set. The validation set is used to determine the best value of $\textbf{c}$ found over all the iterations using different initializations. We iterate 25 times in the optimization procedure of $\mathbf{c}$. We fixed $||\alpha|| = 1$ in the optimization problem of combination operator. The optimal choice of combination operator is powerful with clear improvements in classification accuracies over the uniform weight approach.

The results demonstrate that the Nystr{\"o}m type subsamling can be effectively used in order to reduce the complexity of kernel methods. The aggregation approach automatically produces the best approximant. In all the cases shown in the tables, we obtain better results by the aggregation approach compared to the Nystr{\"o}m approximants.

\subsection{NSL-KDD data set}
We consider the NSL-KDD benchmark data set \cite{NSL-KDD} from UCI machine learning repository for an empirical analysis of Nystr{\"o}m type subsampling. NSL-KDD data set is the refined version of KDD Cup 99 data set which is used in the $3^{rd}$ International Knowledge Discovery and Data Mining Tools Competition for intrusion detection. The traning set of NSL-KDD data set does not include redundant records and contains reasonable number of records to run the experiments on the complete set. The NSL-KDD data contains 41 attributes and 5 classes that are normal and 4 categories of attacks: DoS, Probe, R2L and U2R. Denial of Service Attack (DoS) is an attack in which the attacker attempts to block system or network resources and services. In Probing Attack, the attacker attempts to gain the information about potential vulnerabilities of a network of computers for the apparent purpose of circumventing its security controls. User to Root Attack (U2R) is a class of exploit in which the attacker access the system as a normal user and break the vulnerabilities to gain administrative privileges. Remote to Local Attack (R2L) occurs when an attacker who has unauthorized ability to dump data packets to remote system over network and exploits some vulnerability to gain access either as a user or root to do their unauthorized activity.

The features in NSL-KDD data set can be classified into three groups: (a) {\bf the basic input features} encapsulate all the attributes that can be extracted from a TCP/IP connection. It includes some flags in TCP connections, duration, prototype, number of bytes from source IP addresses or from destination IP addresses and service, (b) {\bf the content input features} use domain knowledge to assess the payload of the original TCP packets. It includes features such as the number of failed login attempts and (c) {\bf the statistical input features} that are determined either by a time window or a window of certain kind of connections. It examines only the connections in the past 2 seconds that have the same destination host or service as the current connection.

In our experiment, we used first $25000$ patterns from ``20 Percent Training Set" available on \cite{NSL-KDD}. NSL-KDD data set contains numeric and nonnumeric attributes. First, we convert the nonnumeric attributes: protocol\_type, service and flag as numeric attributes.  We represent the corresponding values of individual strings ``tcp", ``udp", ``icmp" for protocol name as $0, 1, 2,$ respectively . All the remaining attributes are also represented according to the above encoding system. We removed the two columns of attribute which contain only zero values. The attributes of the input data are normalized to the interval $[0,1]$ (see \cite{Alom}).

We consider a multi-penalty regularization scheme based on Nystr{\"o}m type subsampling which can be viewed as a special case of proposed problem (\ref{multi.plty.funcl1}),
$$\mathop{\text{arg}\min}_{f \in \HH^{\xx_s}} \left\{\frac{1}{m}\sum\limits_{i=1}^m(f(x_i)-y_i)^2+\la_0||f||_\HH^2+\la_1||(S_{\xx}^*LS_{\xx})^{1/2}f||_\HH^2\right\},$$
where the Laplacian $L=D-W$ with $W=(\omega_{ij})$ is a weight matrix with non-negative entries and $D$ is a diagonal matrix with $D_{ii}=\sum\limits_{j=1}^m \omega_{ij}$. We apply this scheme on NSL-KDD data set for binary classification among attacks and normal situations. We demonstrate the performance of single-penalty regularization ($\la_1=0$) versus multi-penalty regularization and also describe the efficiency of proposed regularization algorithm statistically using the standard error measures. All regularized solutions appearing in this experiment are constructed in the reproducing kernel Hilbert space corresponding to the Gaussian kernel $K(x_i,x_j)=\exp(-\gamma||x_i-x_j||^2)$ with the exponential weights $\omega_{ij}=\exp(-||x_i-x_j||^2/{4b})$, for some $b,\gamma>0$. We choose the regularization parameters for single-penalty regularization according to the balancing principle \cite{DeVito} and for multi-penalty regularization according to the balanced-discrepancy principle \cite{Abhishake}.

The performance of the proposed approach is evaluated using various performance measures. To measure the performance of learning classifiers we need to know the terms: True Positive TP (the number of correctly classified positive instances), False Negative FN (the number of misclassified positive instances), False Positive FP (the number of misclassified negative instances) and True Negative TN (the number of correctly classified negative instances). Standard performance measures: Classification accuracy, Precision, Sensitivity, Specificity and F-measure can be defined as:
$$\text{Classification accuracy} =\frac{\text{TP+TN}}{\text{TP+TN+FP+FN}},$$
$$\text{Precision} =\frac{\text{TP}}{\text{TP+FP}},$$
$$\text{Sensitivity (True positive rate)} =\frac{\text{TP}}{\text{TP+FN}},$$
$$\text{Specificity (True negative rate)} =\frac{\text{TN}}{\text{FP+TN}}$$
and
$$\text{F-measure} =\frac{2\times \text{Precision}\times \text{Sensitivity}}{\text{Precision}+\text{Sensitivity}} =\frac{\text{2TP}}{\text{2TP+FN+FP}}$$

We use the 10-fold cross validation which divides $25000$ patterns of the considered data set into 10 sub-data sets of size $2500$. We train our algorithm on first 9 sub-data sets and test on the last sub-data set. In our experiment, we illustrate the performance of multi-penalty regularized solutions $f_{\zz,\la}^{10}$, $f_{\zz,\la}^{50}$, $f_{\zz,\la}^{250}$ corresponding to subsampling size $s=10, 50, 250$ and their aggregated solution $f_{\zz}$ for $|\zz|=2500$. We also compare the performance of these estimators with single-penalty regularized solution $f_{\zz,\la_0}$ and multi-penalty regularized solution $f_{\zz,\la}$. In the context of binary classification, the output $y_i$ takes only two values, designated by $1$ for attack and $-1$ for normal. For the regularized solutions we consider the decision rule/classisfier $\{y=1 \text{ for } f(x)\geq0 \text{ and } y=-1 \text{ for } f(x)<0\}$ in discriminating the elements x of two classes. In the experiments, the initial parameters are $\la_0=10^{-8},~\la_1=1$, the kernel parameter $\gamma=4\times 10^{-2}$ and the weight parameter $b=10^{-3}$.

\begin{table}[h!]
\begin{center}
\begin{tabular}{| m{1.5em} | m{1cm}| m{1cm} | m{1cm} | m{1cm}| m{1cm} | m{1cm} | m{1cm}| m{1cm} | m{1cm} |}
\hline
\begin{turn}{+90}Estimators~~~\end{turn}& Fold1 (\%)        & Fold2 (\%)        & Fold3 (\%)         & Fold4 (\%)         & Fold5 (\%)         & Fold6 (\%)         & Fold7 (\%)       & Fold8 (\%)       & Fold9 (\%)\\
\hline
$f_{\zz,\la_0}$    & \text{ }97.80          & \text{ }83.48         & \text{ }98.56         & \text{ }98.56          & \text{ }98.44         & \text{ }98.40          & \text{ }80.20         & \text{ }98.64         & \text{ }98.12 \\
\hline
$f_{\zz,\la}$      & \text{ }98.96          & \text{ }98.64          & \text{ }98.48          & \text{ }98.48          & \text{ }98.36          & \text{ }98.44          & \text{ }98.84          & \text{ }98.84          & \text{ }98.08 \\
\hline
$f_{\zz,\la}^{10}$ & \text{ }92.52 (0.16) & \text{ }92.48 (0.18) & \text{ }92.57 (0.20) & \text{ }92.64 (0.20) & \text{ }92.47 (0.15) & \text{ }92.47 (0.20) & \text{ }92.65 (0.15) & \text{ }92.06 (0.23) & \text{ }92.76 (0.17) \\
\hline
$f_{\zz,\la}^{50}$ & \text{ }95.79 (0.07) & \text{ }96.45 (0.05) & \text{ }96.19 (0.06) & \text{ }96.48 (0.07) & \text{ }95.88 (0.06) & \text{ }95.92 (0.06) & \text{ }96.17 (0.08) & \text{ }96.15 (0.07) & \text{ }95.84 (0.07) \\
\hline
$f_{\zz,\la}^{250}$& \text{ }98.33 (0.02) & \text{ }98.31 (0.03) & \text{ }98.03 (0.02) & \text{ }98.36 (0.02) & \text{ }98.41 (0.03) & \text{ }98.18 (0.03) & \text{ }98.58 (0.02) & \text{ }98.06 (0.03) & \text{ }98.56 (0.03) \\
\hline
$f_{\zz}$          & \text{ }98.33 (0.02) & \text{ }98.32 (0.03) & \text{ }98.03 (0.02) & \text{ }98.37 (0.02) & \text{ }98.42 (0.02) & \text{ }98.19 (0.03) & \text{ }98.60 (0.02) & \text{ }98.06 (0.03) & \text{ }98.57 (0.02) \\
\hline
\end{tabular}
\caption{\small Statistical performance of various estimators on different sub-data sets (Folds) of NSL-KDD data set using random subsampling 50 times.} \label{Table.NSL-KDD.rand}
\end{center}
\end{table}

\begin{table}[h!]
\begin{center}
\begin{tabular}{ | m{1.5em} | m{1.45cm}| m{1.55cm} | m{1.7cm} | m{1.7cm}| m{1.9cm} | m{2.9cm} |}
  \hline
\begin{turn}{+90}Estimators~~~\end{turn}& Accuracy \text{    }(\%) & Precision & Sensitivity & Specificity & F-measure & Parameter choice \\
  \hline
  $f_{\zz,\la_0}$      & \text{ }94.32 (3.79) & \text{ }0.93314 (0.04480) & \text{ }0.98539 (0.00330) & \text{ }0.90603 (0.07239) & \text{ }0.95220 (0.02842) & $\la_0=5.52\times 10^{-7}$ \\
  \hline
  $f_{\zz,\la}$        & \text{ }98.56 (0.10) & \text{ }0.98100 (0.00163) & \text{ }0.98852 (0.00121) & \text{ }0.98311 (0.00148) & \text{ }0.98474 (0.00105) & $\la_0=5.52\times 10^{-7}$ $\la_1=4.31\times 10^{-3}$ \\
  \hline
  $f_{\zz,\la}^{10}$   & \text{ }92.29 (0.30) & \text{ }0.93164 (0.00279) &  \text{ }0.90179 (0.00734) & \text{ }0.94156 (0.00276) & \text{ }0.91628 (0.00358) & $\la_0=1.69\times 10^{-8}$ $\la_1=5.71\times 10^{-6}$ \\
  \hline
  $f_{\zz,\la}^{50}$   & \text{ }96.29 (0.12) & \text{ }0.97816 (0.00132) & \text{ }0.94193 (0.00318) & \text{ }0.98144 (0.00118) & \text{ }0.95967 (0.00140) & $\la_0=2.25\times 10^{-7}$ $\la_1=5.66\times 10^{-5}$ \\
  \hline
  $f_{\zz,\la}^{250}$  & \text{ }98.33 (0.09) & \text{ }0.98171 (0.00085) & \text{ }0.98273 (0.00135) & \text{ }0.98386 (0.00075) & \text{ }0.98222 (0.00098) & $\la_0=8.42\times 10^{-8}$ $\la_1=4.72\times 10^{-4}$ \\
  \hline
  $f_{\zz}$            & \text{ }98.33 (0.09) & \text{ }0.98171 (0.00085) & \text{ }0.98273 (0.00135) & \text{ }0.98386 (0.00075) & \text{ }0.98222 (0.00098) & \\
  \hline
\end{tabular}
\caption{\small Statistical performance of various estimators over all 9 sub-data sets (Folds) of NSL-KDD data set.}\label{Table.NSL-KDD.whole}
\end{center}
\end{table}

\begin{table}[h!]
\begin{center}
\begin{tabular}{| m{32em} | m{1.3cm}|}
  \hline
  Hybrid Naive based classifier using 2500 (approximately) patterns \cite{Farid} & 82.39\% \\
  \hline
  DBN-SVM on normalized encoded data set using 5000 patterns \cite{Alom} & 96.90\% \\
  \hline
  Cross-breed type Bayesian network on 2500 patterns \cite{Onik} & 97.27\% \\
  \hline
  Multi-penalty (Manifold) regularization algorithm on 2500 patterns \cite{Abhishake} & {\bf 98.56\%} \\
  \hline
  Proposed multi-penalty kernel method based on Nystr{\"o}m type subsampling on 2500 patterns & {\bf 98.33\%} \\
  \hline
\end{tabular}
\caption{\small Performance comparison of the proposed model with some existing research work on NSL-KDD data set.}\label{Table.NSL-KDD.comparision}
\end{center}
\end{table}

Note that the performance of the Nystr{\"o}m type subsampling depends not only on the size $s$ of a subsampling set but also on the sub-data set $\{(x_i,y_i)\}_{i=1}^s$. To demonstrate reliability of Nystr{\"o}m type subsampling we generate 50 times subsamples of size $s=10, 50, 250$ from 9 sub-data sets (Folds) and report the mean and standard deviation of performance accuracy over 9 sub-data sets in Table \ref{Table.NSL-KDD.rand}. We observe that the performance of the single-penalty regularization varies over the different folds. On the other hand, multi-penalty regularized solution performs consistently. The accuracy of the multi-penalty regularized solutions varies over the subsampling size $s$ but the aggregated solution based on linear functional strategy almost provides the better solution among the Nystr{\"o}m approximants. Here we also note that the performance of the Nystr{\"o}m approximants does not vary much over the different random subsampling of same size. In Table \ref{Table.NSL-KDD.whole}, we report the results in terms of mean and standard deviation of the performance measures over nine sub-data sets with the balanced-discrepancy parameter choice.

In Table \ref{Table.NSL-KDD.comparision}, we compare the performance of proposed approach with existing approaches applied on NSL-KDD data set for intrusion detection. The multi-penalty regularization scheme illustrates the better performance than the benchmark result achieved in \cite{Onik}. Moreover, the proposed multi-penalty kernel method based on Nystr{\"o}m type subsampling is able to achieve the more accurate results. The reported results of the experiments demonstrate that the aggregation approach automatically uses the best Nystr{\"o}m approximant and achieves the accuracy of standard multi-penalty regularization scheme. The results shows that the Nystr{\"o}m subsampling can yield very good performance while substantially reducing the computational requirements. So the approach can be efficiently applied as a reliable strategy when dealing with the data of big size.

\bibliography{Bib_file}
\bibliographystyle{plain}
\end{document}